\documentclass{article} % For LaTeX2e
\usepackage{nips15submit_e,times}
\usepackage{hyperref}
\usepackage{url}

\usepackage[T1]{fontenc}
\usepackage{times}%
\usepackage[cmex10]{amsmath}
\usepackage{algorithmic}
\usepackage{url}
\usepackage{amssymb}
\usepackage{amsthm}
\usepackage{color}
\usepackage{bm}
\usepackage{colortbl}
\usepackage{graphicx}
\usepackage{algorithm} %format of the algorithm
\usepackage{threeparttable}

\usepackage{stfloats}

\usepackage{subfig}
\newtheorem{thm}{Theorem}[section]
\newtheorem{cor}[thm]{Corollary}

\newtheorem{lem}[thm]{Lemma}

\newtheorem{remark}[thm]{Remark}
\newtheorem{Def}{Definition}

\renewcommand{\v}{{\boldsymbol{v}}}

\newcommand{\w}{{\boldsymbol{w}}}

\def\tr{{\mathrm{tr}}}
\title{Stochastic Variance Reduction Gradient for a Non-convex Problem Using Graduated Optimization}

\author{
Li Chen~~~~~~~~Shuisheng Zhou\thanks{Corresponding author: sszhou@mail.xidian.edu.cn}~~~~~~~~Zhuan Zhang\\
School of Mathematics and Statistics, Xidian University, Xi'an, 710071 China. \\
}

\nipsfinalcopy % Uncomment for camera-ready version

\begin{document}

\maketitle

\begin{abstract}
In machine learning, nonconvex optimization problems with multiple local optimums are often encountered. Graduated Optimization Algorithm (GOA) is a popular heuristic method to obtain global optimums of nonconvex problems through progressively minimizing a series of convex approximations to the nonconvex problems more and more accurate. Recently, such an algorithm GradOpt based on GOA is proposed with amazing theoretical and experimental results, but it mainly studies the problem which consists of one nonconvex part. This paper aims to find the global solution of a nonconvex objective with a convex part plus a nonconvex part based on GOA. By graduating approximating non-convex part of the problem and minimizing them with the Stochastic Variance Reduced Gradient (SVRG) or proximal SVRG, two new algorithms, SVRG-GOA and PSVRG-GOA, are proposed. We prove that the new algorithms have lower iteration complexity ($O(1/\varepsilon)$) than GradOpt ($O(1/\varepsilon^2)$). Some tricks, such as enlarging shrink factor, using project step, stochastic gradient, and mini-batch skills, are also given to accelerate the convergence speed of the proposed algorithms. Experimental results illustrate that the new algorithms with the similar performance can converge to 'global' optimums of the nonconvex problems, and they converge faster than the GradOpt and the nonconvex proximal SVRG.
\end{abstract}

\section{Introduction}
In machine learning, we often encounter the following optimization problem:
\begin{equation}\label{eq:problem}
\min\limits_{\mathbf{w}} F(\mathbf{w})=h(\mathbf{w})+f(\mathbf{w}),
\end{equation}
where $h:\mathbb{R}^d\rightarrow \mathbb{R}$ is a convex function. Problem \eqref{eq:problem} encompasses a wide variety of problems
which have been studied in many different areas, including image restoration \cite{selesnick2015convex}, pattern recognition \cite{laporte2014nonconvex}, and compressed sensing \cite{donoho2006compressed}.
%In finite-sum optimization problems, $f(\mathbf{w})$ can be represented by
%\begin{equation}\label{eq:f_sum}
%f(\mathbf{w})=\frac{1}{n}\sum _{i=1}^n f_i(\mathbf{w}),
%\end{equation}
%where $f_i:\mathbb{R}^d\rightarrow \mathbb{R}$ is non-convex.
 In regularized loss minimization problems, $h(\mathbf{w})$ and $f(\mathbf{w})$ could be considered as the regularization term and the loss term respectively. For example, given a training set $\{(\mathbf{x}_i,y_i)\}_{i=1}^n$, where $\mathbf{x}_i\in \mathbb{R}^d$ and $y_i\in\mathbb{R}$, $F(\mathbf{w})=\frac{\lambda}{2}\| \mathbf{w}\|_2^2+\frac{1}{n}\sum\limits_{i=1}^n L(\mathbf{w}^\top \mathbf{x}_i-y_i)$ is the support vector machine (SVM), where $\lambda$ is the positive regularization parameter, and $L(\cdot)$ is called the loss function of the sample $(\mathbf{x}_i,y_i)$.

 If $f(\mathbf{w})$ is convex, the ordinary convex optimization methods, such as gradient, dual, and so on, can solve Problem \eqref{eq:problem}. Furthermore, if $f(\mathbf{w})$ is represented by limit-sum (i.e. $f(\mathbf{w})=\frac{1}{n}\sum_{i=1}^nf_i(\mathbf{w})$), some stochastic methods, such as Stochastic Gradient Descent (SGD) and Proximal SGD (Prox-SGD) can be adopted. Recent progress is the variance reduced stochastic methods, such as SVRG (stochastic variance reduced gradient)\cite{johnson2013accelerating}, SAG (stochastic average gradient) \cite{roux2012stochastic, schmidt2013minimizing}, SDCA (stochastic dual coordinate ascent) \cite{shalev2013stochastic} and SAGA \cite{defazio2014saga}. There are also some improved methods including proximal stochastic based methods \cite{johnson2013accelerating, shalev2014accelerated, li2015accelerated, zhang2015stochastic, allen2016even, allen2016katyusha} proposed in the past few years.

  However, in computer vision and machine learning, we are more interested in non-convex functions $f(\mathbf{w})$
  %, for instance, the ramp loss linear programming SVM \cite{huang2014ramp} $F(\mathbf{w})=\frac{\lambda}{2}\| \mathbf{w}\|_1+\frac{1}{n}\sum\limits_{i=1}^n\min\{\max\{\mathbf{w}^\top \mathbf{x}_i-y_i,0\},1\}$,
  due to its special advantages and extensive applications. For example, using non-convex loss function, such as truncated hinge loss \cite{wu2007robust}, ramp loss \cite{liu2016ramp} and robust loss \cite{feng2016robust}, one can reduce the influence of noise to models. Moreover, the deep neural networks are also the highly non-convex optimization problems.
%For non-convex problems, researches are relatively fewer than the convex one.

At present, more and more people focus on solving non-convex problems. In terms of the case that $f(\mathbf{w})$ is a limit-sum of non-convex functions, H. Li \cite{li2015accelerated} and S. Ghandimi \cite{ghadimi2016accelerated} provide accelerated GD and Prox-GD, and show that these methods can converge if the parameters are tuned properly. The results of S. J. Reddi \cite{reddi2016proximal, reddi2016stochastic} and Z. A. Zhu \cite{allen2016variance} indicate that SVRG and Proximal SVRG (Prox-SVRG) can be used to solve non-convex finite-sum problems, and show that they convergent faster than SGD and GD. But our experimental results illustrate that SVRG and Prox-SVRG may not converge to the global optimization for non-convex functions.

Graduated optimization algorithm (GOA) \cite{blake1987visual,ye2003estimating,gashler2011manifold} is a global searching algorithm for nonconvex problems. It starts from an initial estimate and progressively minimizes a series of finer and simpler approximations to the original problem. In these sequences, if the solution to the previous problem falls within the locally convex region around the solution to the next problem, then the algorithm will find the globally optimal solution to the nonconvex optimization problem at the end of the sequence optimization. There are some ways to progressively deform the nonconvex objective to some convex task. One possible principle is by Gaussian smoothing \cite{duchi2012randomized} $\widehat{F}_\delta(\mathbf{w})=\mathbf{E}_u[ F(\mathbf{w}+\delta u)]$, where $u$  follows Gaussian distribution. Gaussian smoothing constructs a collection of functions ranging from a highly smoothed to the actual nonconvex function by adjusting $\delta$ from high to low.

 Hazan et al. \cite{hazan2016graduated} proposed the GradOpt method based on GOA and Suffix-SGD for a class of non-convex functions $(a,\sigma)$-nice. GradOpt constructs a series of local strong convex functions with $u$ according to uniform distributions on norm balls.
 %These local strong convex functions are finer and finer approximate to the original problem. The solution to each problem is within the convex region around the optimum of the next problem.
 Then Suffix-SGD is used to solve these local strong convex functions efficiently. Hazan et al. \cite{hazan2016graduated} prove that GradOpt with proper parameters is able to converge to a global optimum for $(a,\sigma)$-nice functions, and give the convergence rate of GradOpt. Experimental results in \cite{hazan2016graduated} show that GradOpt is faster and yields a much better solution than Minibatch SGD. However, our experimental results illustrate that GradOpt has some shortcomings, for example, the Suffix-SGD which is used in GradOpt has slow convergence due to the inherent variance, the smooth version is far from the original function at the initial steps and the conditions of the definition of $(a,\sigma)$-nice are so strong that the application range of GradOpt is limited. For the non-$(a,\sigma)$-nice functions, GradOpt converges gradually slow with the increase of iterations and may terminate before finding the global optimum.

In this paper, our proposed methods overcome the shortcomings of the GradOpt. We study a class of non-convex optimization problems: a sum of a convex function and a nonconvex function which has multiple local optimums, and propose two low-complexity iteration algorithms to fast converge to the \textbf{global optimum} of such non-convex optimization problems.
%In the new algorithms, our convex approximating strategy can obtain a local strong convex function which is closer to original problem than the smoothing strategy used in GradOpt. Because the variance of suffix-SGD is large due to random sampling, we design a improved SVRG method which can reduce the variance of random sampling and avoid iteration points stepping out of a bound. Moreover, in GradOpt, the shrinkage factor is fixed at 0.5, which may lead to the algorithm not converging to the global optimum. In comparison, the shrinkage factor in our algorithms can be adjust, which avoid the algorithms terminating prematurely.

\textbf{Main Contributions.}
Our main contributions are stated below:
\begin{itemize}
\item SVRG/prox-SVRG is applied in the new algorithms instead of the Suffix-SGD as used in GradOpt. In the SVRG/Prox-SVRG in our algorithms, a projection step is added to avoid iteration points stepping out of a bound.
\item The new algorithms are proved that they have the complexity (the number of iterations to obtain a $\varepsilon$-accuracy solution) $O(1/\varepsilon)$.  This is far superior to the complexity of GradOpt $O(1/\varepsilon^2)$ \cite{hazan2016graduated}.
\item By introducing a shrink parameter, the new definition $(a,c,\sigma)$-nice is proposed. Moreover, we design a better convex approximation method for the non-convex model \eqref{eq:problem} than $\delta$-smooth in GradOpt.
     %Moreover, through choosing a relatively larger shrinkage factor, the new approaches can fast converge to the global optimums, and avoid terminating before finding the global minimum.
\end{itemize}

This paper is structured as follows: Section 2 gives some definitions used in this paper. Section 3 and 4 present two new algorithms and their theorem analyses. Section 5 proposes two extensions for our algorithms. Section 6 describes the experimental results. Section 7 concludes this paper and points out future research directions.

\section{Setting and Definitions}
\textbf{Notation}: During this paper, we use $\mathbb{B}$ to denote the unit Euclidean ball in $\mathbb{R}^d$ and $\mathbb{B}_r(\mathbf{w})$ as the Euclidean $r$-ball in $\mathbb{R}^d$ centered at $\mathbf{w}$. $u\sim \mathbb{B}$ denotes a random variable distributed uniformly over $\mathbb{B}$. Then we have $\mathbf{E}_{u\sim\mathbb{B}}[u]=0$ since the uniform ball smoothing with parameter $\delta$ is equivalent to (zero mean) Gaussian smoothing \cite{hazan2016graduated} and let $\mathbf{E}_{u\sim\mathbb{B}}[uu^\top]=\varrho^2 I$. For convenience, $\|\cdot\|^2$ denotes $\|\cdot\|_2^2$ in this paper.
\subsection{Common Definitions}
Recall the definitions of strong-convex and $L$-Lipschitz functions as follows \cite{shalev2014understanding}.
\begin{Def}
($\sigma$-strongly-convex) A function $f:\mathbb{R}^d\rightarrow \mathbb{R}$ is said to be $\sigma$-strongly convex over a set $\mathcal{C}$ if for any $\mathbf{w}_1, \mathbf{w}_2\in\mathcal{C}$ the following holds,
  \begin{equation*}
  f(\mathbf{w}_2)-f(\mathbf{w}_1)\geq \nabla f(\mathbf{w}_1)^\top(\mathbf{w}_2-\mathbf{w}_1)+\frac{\sigma}{2}\|\mathbf{w}_1-\mathbf{w}_2\|^2.
  \end{equation*}
\end{Def}

\begin{Def}
($L$-smoothness) A differentiable function $f:\mathbb{R}^d\rightarrow \mathbb{R}$ is $L$-smooth if its gradient is $L$-Lipschitz; namely, for all $\mathbf{w}_1, \mathbf{w}_2$ we have $\|\nabla f(\mathbf{w}_1)-\nabla f(\mathbf{w}_2)\|\leq L\|\mathbf{w}_1-\mathbf{w}_2\|$.
\end{Def}
\subsection{Partial $\delta$-smooth}\label{sec:smooth}
To construct finer and finer approximations to the original objective function, Paper \cite{hazan2016graduated} defines $\delta$-smooth by adopting the uniform ball smoothing strategy \cite{duchi2012randomized}. According to $\delta$-smooth, $L$-Lipschitz nonconvex function $F(\mathbf{w})$ is smoothed as $$\hat {F}_\delta (\mathbf{w})=\mathbf{E}_{u\sim \mathbb{B}}[F(\mathbf{w}+\delta u)].$$%=\mathbf{E}_{u\sim \mathbb{B}}[h(\mathbf{w}+\delta u)]+\mathbf{E}_{u\sim \mathbb{B}}[f(\mathbf{w}+\delta u)]
 where $\mathbf{E}(\cdot)$ denotes expectation, $\delta$ is the smooth parameter.

 $\delta$-smooth can transform a nonconvex function into a local strong convex function. But the smoothed version is a little far from the original function, see Fig. \ref{fig:1dim}. Fig. \ref{fig:1dim} gives a 1-dim non-convex function: $y=\frac{w^2}{2}-0.3[\exp\frac{-(w-1)^2}{0.02}-\exp\frac{-(w+1.3)^2}{0.045}]$ and its different convex smoothed versions with $\delta=1$. In Fig. \ref{fig:1dim}, $\delta$-smooth is a little far from the original function. The primary cause is $\frac{1}{2}\mathbf{E}_{u\sim \mathbb{B}}[({w}+\delta u)^2]=\frac{1}{2}w^2+\frac{1}{2}\delta^2 \varrho^2$, here we used $\mathbf{E}_{u\sim \mathbb{B}}[u]=0$. To make the smoothed version closer to the original function, we extend the $\delta$-smooth and give our Partial $\delta$-smooth below.
 \begin{Def}\label{def:smooth}
(Partial $\delta$-smooth) Let $\delta \geq 0$. Given an $L$-smoothness function $F:\mathbb{R}^d\rightarrow \mathbb{R}$ define its Partial $\delta$-smooth function to be
   \begin{equation*}
   \hat {F}_\delta (\mathbf{w})=h(\mathbf{w})+\hat{f}_\delta (\mathbf{w})=h(\mathbf{w})+\mathbf{E}_{u\sim \mathbb{B}}[f(\mathbf{w}+\delta u)].
   \end{equation*}
\end{Def}

 \begin{remark}
The reason Partial $\delta$-smooth is the convex proximate of the original function is as follows.
  \begin{equation*}
   \begin{split}
   &\hat {F}_\delta (\mathbf{w})=h(\mathbf{w})+\mathbf{E}_{u\sim\mathbb{B}}[f(\mathbf{w}+\delta u)]\\
   &\approx h(\mathbf{w})+\mathbf{E}_{u\sim\mathbb{B}}[f(\mathbf{w})+\delta \nabla f(\mathbf{w})^\top u+\frac{\delta^2}{2}u^\top\nabla^2 f(\mathbf{w})u]\\
   &=h(\mathbf{w})+f(\mathbf{w})+\frac{\delta^2}{2}\mathbf{E}_{u\sim\mathbb{B}}[u^\top\nabla^2 f(\mathbf{w})u]\\
   &=h(\mathbf{w})+f(\mathbf{w})+\frac{\delta^2}{2}\mathbf{E}_{u\sim\mathbb{B}}[\tr(\nabla^2 f(\mathbf{w})u u^\top)]\\
   &=h(\mathbf{w})+f(\mathbf{w})+\frac{\delta^2}{2}\tr[\nabla^2 f(\mathbf{w})\mathbf{E}_{u\sim\mathbb{B}}(u u^\top)]\\
   &=F(\mathbf{w})+\frac{\delta^2}{2}\tr(\nabla^2 f(\mathbf{w}))\tr[\mathbf{E}_{u\sim\mathbb{B}}(u u^\top)]\\
   &=F(\mathbf{w})+\frac{\delta^2\varrho^2}{2}\tr(\nabla^2 f(\mathbf{w})).
   \end{split}
  \end{equation*}
    From the last equation, we can obtain the conclusion that the function values around the local minima become larger as $tr(\nabla^2 f(\mathbf{w}))\geq 0$. Similarly, function values around the local maxima become smaller as $tr(\nabla^2 f(\mathbf{w}))\leq 0$. So the original function is smoothed.
 \end{remark}

\begin{figure}[!htb]
\centering
\includegraphics[width=0.8\textwidth]{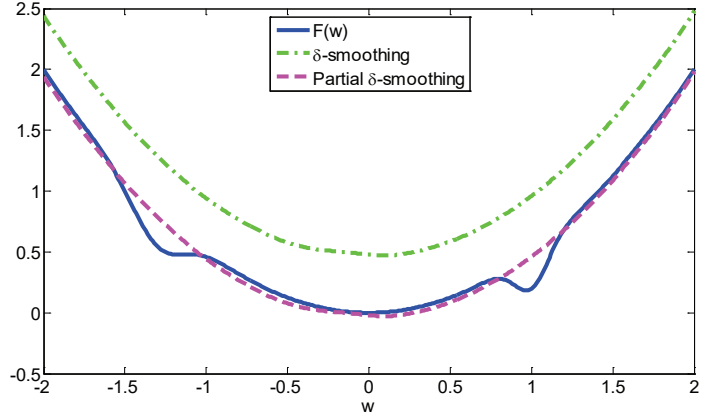}
\caption{Nonconvex function $F(\mathbf{w})=\frac{w^2}{2}-0.3[\exp\frac{-(w-1)^2}{0.02}-$ $\exp\frac{-(w+1.3)^2}{0.045}]$, and its two smoothed versions, $\delta=1$.}
\label{fig:1dim}
\end{figure}%

In Def. \ref{def:smooth}, the smaller $\delta$ is, the better approximation is for $f(\mathbf{w})$. If $\delta=0$, then $\hat{f}_\delta (\mathbf{w})=f(\mathbf{w})$ and $\hat{F}_\delta (\mathbf{w})=F(\mathbf{w})$. Partial $\delta$-smooth only smoothes the nonconvex part $f(\mathbf{w})$, but remains the convex part $h(\mathbf{w})$ unchanged for $F(\mathbf{w})=h(\mathbf{w})+f(\mathbf{w})$. By comparison, in $\delta$-smooth, both convex part and nonconvex part are smoothed. Fig. \ref{fig:1dim} confirms the effectiveness of our strategy. From Fig. \ref{fig:1dim}, it is easy to observe that the function yielded by Partial $\delta$-smooth is closer to original function than by $\delta$-smooth, and the functions yielded by these two methods are both local strong convex functions.
\subsection{$(a,c,\sigma)$-nice Functions} \label{sec:multimodal}
Hazan et al. \cite{hazan2016graduated} define a class of nonconvex functions $(a,\sigma)$-nice (here $a,\sigma>0$) and use GradOpt algorithm obtain the global optimum of $(a,\sigma)$-nice functions. In this subsection, we denote $\mathbf{w}_\delta^\ast\in \arg\min_{\mathbf{w}\in \mathcal{C}}\hat{F}_{a\delta}(\mathbf{w})$. For every $\delta>0$, $(a,\sigma)$-nice functions require $\hat{F}_{a\delta}(\mathbf{w})$ to be $\sigma$-strongly-convex in $\mathbb{B}_{3\delta}(\mathbf{w}_\delta^\ast)$ and $\|\mathbf{w}_{\delta}^\ast-\mathbf{w}_{\delta/2}^\ast\|\leq \delta/2$. The conditions in $(a,\sigma)$-nice functions are so strong that the application range of GradOpt is limited.
%Besides, in GradOpt algorithm, $\delta$ declines by half in each iteration, which causes that the $\mathbb{B}_{3\delta}(\mathbf{w}_\delta^\ast)$ shrinks too fast and GradOpt may converge to local optimum or terminate before finding a optimum.
  To weaken the conditions of $(a,\sigma)$-nice functions and enlarge the application range of the GOA, we propose the following definition.
\begin{Def}\label{def:multimodal}
($(a,c,\sigma)$-nice) Denote $1/2\leq c<1$ as shrink factor, let $a,\sigma>0$. A function $F:\mathcal{C}\rightarrow \mathbb{R}$ is said to be $(a,c,\sigma)$-nice if, for every $\delta>0$, there exists $\mathbf{w}_{c\delta}^\ast$, such that $\|\mathbf{w}_{\delta}^\ast-\mathbf{w}_{c\delta}^\ast\|\leq c\delta$, and $\hat{F}_{a\delta}(\mathbf{w})$ is $\sigma$-strong-convex over $\mathbb{B}_{r\delta}(\mathbf{w}_\delta^\ast)$, where $r\geq1.5$.
\end{Def}

\begin{figure}[!htb]
\centering
\includegraphics[width=0.8\textwidth]{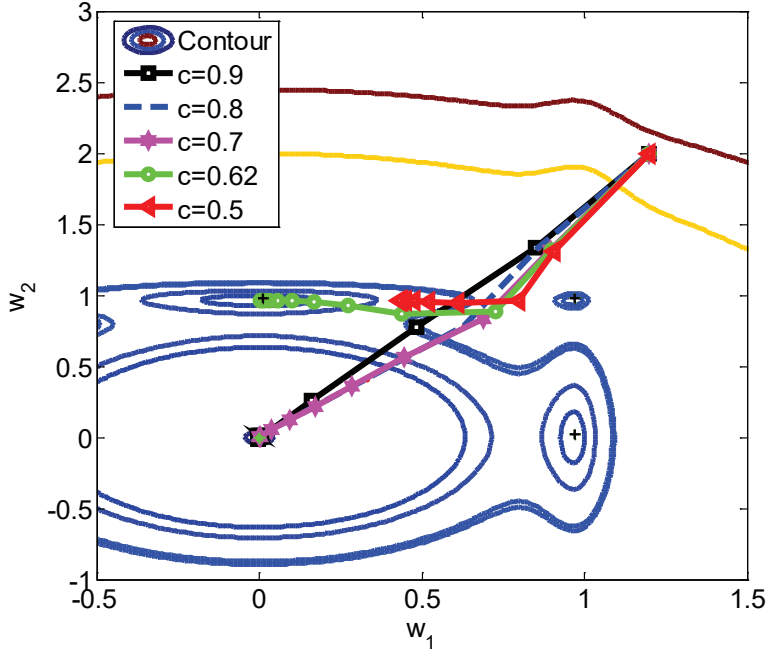}
\caption{GradOpt with varying $c$ values, '$\textcolor{blue}{\ast}$' $-$ init point, '$\color{green}{\ast}$' $-$ global optimum, '$+$' $-$ local optimums. There is only one global optimum and three local optimums.}
\label{fig:contour}
\end{figure}

$(a,c,\sigma)$-nice is the main type of functions that we discuss in this article. When we let $c=1/2$ and $r=3$, then $(a,c,\sigma)$-nice functions are similar to the $(a,\sigma)$-nice functions. $(a,c,\sigma)$-nice functions imply that the optimum of $\hat{F}_{a\delta}(\mathbf{w})$ is a good start for optimizing a finer approximated version $\hat{F}_{ca\delta}(\mathbf{w})$, and for every $\delta>0$, $\hat{F}_{a\delta} (\mathbf{w})$ is local strong convex in $\mathbb{B}_{r\delta}(\mathbf{w}_\delta^\ast)$. So GOA can be adopted to obtain the global minimization of $(a,c,\sigma)$-nice functions.

In $(a,c,\sigma)$-nice functions, it would be better to choose a larger $c$, such as $c=0.9$, because small $c$ may lead to an algorithm converging slowly after some iterations or termination before finding the global optimum, see Fig. \ref{fig:contour}.
 Fig. \ref{fig:contour} gives
 %the contour plot of the objective function and
 the results of running the GradOpt algorithm  with different $c$ on optimization problem $\min_{\mathbf{w}\in\mathbb{R}^2} \frac{\|\mathbf{w}\|^2}{2}-\frac{3}{10}[\exp\frac{-(w_1-1)^2}{0.02}-\exp\frac{-(w_2-1)^2}{0.02}]$. $c=1/2$ is corresponding to the results of performing original GradOpt. From Fig. \ref{fig:contour}, it can be shown that when $c$ is set larger, such as $c=0.9$, $c=0.8$ and $c=0.7$, GradOpt can find the global optimum. However, if $c$ is set small, such as $c=0.62$ and $c=0.5$, GradOpt may converge to local optimum or terminate before finding an optimum. In addition, $r$ is set as $3$ in the definition of $(a,\sigma)$-nice functions, whereas Theorem \ref{thm:6.1} proves that $r\geq 1.5$ is enough to guarantee an algorithm to converge to a global optimum point.

 %Gradient descent methods is an efficient approach to get the minimum point of $\hat{F}_{\delta}(\mathbf{w})$ in $\mathbb{B}_{r\delta}(\mathbf{w}_\delta^\ast)$. However, directly compute the gradients of $\hat{F}_\delta (\mathbf{w})$ may be very costly because $\nabla\hat{f}_\delta (\mathbf{w})=\mathbf{E}_{u\sim \mathbb{B}}[\nabla f(\mathbf{w}+\delta u)]$, that is, we need sampling $u$ from $\mathbb{B}$ many times and then calculating expectation of $\nabla f(\mathbf{w}+\delta u)$. An effective approach for calculating $\nabla\hat{F}_\delta (\mathbf{w})$ is to use stochastic gradient method.
 %, that is, sample the gradients of $\hat{F} (\mathbf{w}+\delta u)$ to obtain an unbiased estimate for $\nabla\hat{F}_\delta (\mathbf{w})$.
 \section{Graduated Optimization Algorithm with SVRG}\label{sec:svrg-goa}
  %The idea of GradOpt is that construct a series of local strong convex functions which are more and more close to the original nonconvex function, then Suffix-SGD \cite{rakhlin2011making} is used to obtain optimums of these local strong convex functions in certain areas.
  Suffix-SGD is used in GradOpt. It is an improved method of SGD. At each iteration $k=1,2,\ldots,T$, Suffix-SGD draws $i$ from $\{1,\ldots,n\}$ randomly, and $$\mathbf{w}_{k}=\Pi_\mathcal{C}(\mathbf{w}_{k-1}-\eta_k\nabla f_i(\mathbf{w}_{k-1})),$$ where $\eta_k$ decreases with the increasing of iteration number $k$, $\Pi_\mathcal{C}(\cdot)$ denotes projecting a vector onto an area $\mathcal{C}$. In Suffix-SGD, randomly sampling may lead to large variance which slows down the convergence speed. In order to overcome this shortcoming, we propose Algorithm \ref{alg:a1} based on SVRG.
\subsection{SVRG-GOA}
Algorithm \ref{alg:a1} is an improvement of the GradOpt algorithm \cite{hazan2016graduated}. It is based on the GOA and SVRG. The idea of Algorithm \ref{alg:a1} is that the original non-convex function is approximated by a series of local strong convex functions $\hat{F}_{\delta_m}$ finer and finer for $m=1,\ldots,M$. Then SVRG with project step is adopted on $\mathcal{C}_m$ which is the convex area of $\hat{F}_{\delta_m}$ according to Theorem \ref{thm:6.1} in each iteration.

\begin{thm}\label{thm:6.1}
 Consider $\mathcal{C}_m$ and $\mathbf{w}_{m+1}$ as defined in Algorithm \ref{alg:a1}. Also, denote by $\mathbf{w}_m^\ast$ the minimizer of $\hat F_{\delta_m}$ in $\mathcal{C}_m$. Then for all $1\leq m\leq M$, $\hat F_{\delta_m}$ is $\sigma$-strong convex over $\mathcal{C}_m$ and $\mathbf{w}_m^\ast \in \mathcal{C}_m$.
 \end{thm}

\begin{algorithm}[htp]
\caption{\textbf{\emph{SVRG-GOA}} $-$ \emph{\textbf{G}raduated \textbf{O}ptimization \textbf{A}lgorithm with \textbf{SVRG}}}\label{alg:a1}% 算法的名字
\hspace*{0.02in} {\bf Input:} %算法的输入， \hspace*{0.02in}用来控制位置，同时利用 \\ 进行换行
target error $\varepsilon>0$, step size $0<\eta<1$, decision set $\mathcal{C}$, iterative number $T$, shrink factor $0<c<1$ and $M=\sqrt{1/\varepsilon}\geq1$. Choose initial point $\mathbf{w}_1\in\mathcal{C}$ uniformly at random. Set $\delta_1=\text{diam}(\mathcal{C})$.\\
\hspace*{0.02in} {\bf Output:} %算法的结果输出
$\mathbf{w}_{m+1}$
%\linespread{1.1}\selectfont
\begin{algorithmic}[1]
\FOR{$m=1$ to $M$}
    \STATE Set $\varepsilon_m:=\frac{\sigma c^2\delta_m^2}{8}$, and
     $S=\left(\log{\frac{\sigma \eta(1-2L\eta)T}{1+2L\eta^2\sigma T}}\right)^{-1}\log{\frac{\Delta F_m}{\varepsilon_m}},$ with $\Delta F_m=F(\mathbf{w}_m)-F(\mathbf{w}_\ast)$. \label{cod:s}

    \STATE Set shrunk decision set $\mathcal{C}_m:=\mathcal{C}\cap\mathbb{B}(\mathbf{w}_m,1.5\delta_m)$,

    \STATE $\widetilde{\mathbf{w}}_0=\mathbf{w}_m$;\texttt{//Perform SVRG over $\hat f_{\delta_ m}$}
        \FOR{ $s=1$ to $S$} \label{cod:begin_SVRG}
           \STATE $\widetilde{\mathbf{w}}=\widetilde{\mathbf{w}}_{s-1}$
           \STATE $\widetilde{g}=\mathbf{E}_{u\sim \mathbb{B}}[\nabla f(\widetilde{\mathbf{w}}+\delta_ m u)]$\label{cod:full_gradient}
           \STATE $\mathbf{w}_0=\widetilde{\mathbf{w}}$
            \FOR {$k=1$ to $T$ }
               \STATE    Randomly pick $u\in\mathcal{C}_m$ and update \label{cod:random}
               \STATE    $\v_k=\nabla h(\mathbf{w}_{k-1})+\nabla f(\mathbf{w}_{k-1}+\delta_ m u)-\nabla f(\widetilde{\mathbf{w}}+\delta_ m u)+\widetilde{g}$ \label{cod:v_k}
               \STATE    $\mathbf{w}_k=\Pi_{\mathcal{C}_m}(\mathbf{w}_{k-1}-\eta \v_k)$ \label{cod:prox}
            \ENDFOR
            \STATE set $\widetilde{\mathbf{w}}_{s}=\mathbf{w}_k$ for randomly chosen $k\in\{0,\dots,T-1\}$\label{cod:option1}
           % \State \textbf{option \uppercase\expandafter{\romannumeral2}}: set $\mathbf{w}idetilde{\mathbf{w}}_{s}=\frac{1}{T}\sum_{k=1}^{T}\mathbf{w}_k$ \label{cod:option2}
      \ENDFOR\label{cod:finish_SVRG}
      \STATE $\mathbf{w}_{m+1}=\widetilde{\mathbf{w}}_{s}$
      \STATE $\delta_{m+1}=c\delta_{m}$
\ENDFOR %\State \Return $\mathbf{w}_{m+1}$
\end{algorithmic}
%\vspace*{-3pt}
\end{algorithm}
In Algorithm \ref{alg:a1}, $\widetilde{\mathbf{w}}$ is an estimate of optimum $\mathbf{w}_\ast$. step \ref{cod:begin_SVRG} to step \ref{cod:finish_SVRG} are the variant of SVRG. We add a projection step in step \ref{cod:prox}. The projection is defined as $\Pi_\mathcal{C}(\v):=\arg\min_{\mathbf{w}\in\mathcal{C}}\|\mathbf{w}-\v\|$, $\forall \v\in\mathbb{R}^d$. In terms of the output of the SVRG, we also could set $\widetilde{\mathbf{w}}_{s}=\frac{1}{T}\sum_{k=1}^{T}\mathbf{w}_k$, $\widetilde{\mathbf{w}}_{s}=\mathbf{w}_T$ or $\widetilde{\mathbf{w}}_{s}=\frac{2}{T}\sum_{k=T/2+1}^{T}\mathbf{w}_k$. Experimental results show that these options have similar performances \cite{lin2014}.
\begin{remark}
In Algorithm \ref{alg:a1}, the computational cost of the Step \ref{cod:full_gradient} is high. Because we have to randomly sample $u$ from $\mathbb{B}$ and calculate $\nabla f(\widetilde{\mathbf{w}}+\delta_ m u)$ for many times to obtain $\widetilde{g}$. In practice, $\widetilde{g}$ can be approximately computed by $\nabla f(\widetilde{\mathbf{w}})$, because $\mathbf{E}_{u\sim \mathbb{B}}[\nabla f(\widetilde{\mathbf{w}}+\delta_ m u)]\approx \mathbf{E}_{u\sim \mathbb{B}}[\nabla f(\widetilde{\mathbf{w}})+\delta_ m \nabla^2 f(\widetilde{\mathbf{w}})^\top u]=\nabla f(\widetilde{\mathbf{w}})$.
\end{remark}
 Algorithm \ref{alg:a1} may obtain the global minimum of the $(a,c,\sigma)$-nice functions by the properties of $(a,c,\sigma)$-nice and the principle of GOA. Furthermore, the results of our algorithm are stable because the minimum of every local strong convex function is unique in the convex area. Experimental results in Section \ref{sec:experiment_2} confirm our analysis above.
 % Because the minimizer $\mathbf{w}_m^\ast$ of $\hat{F}_{\delta_m}$ exists in $\mathcal{C}_m$ by Theorem \ref{thm:6.1}, and $\mathbf{w}_m^\ast$ is a good start for optimizing a finer approximated version $\hat{F}_{ca\delta}(\mathbf{w})$.

  Let
\begin{equation}\label{eq:sample}
\begin{split}
\v:=\nabla h(\mathbf{w})+\nabla f(\mathbf{w}+\delta u)-\nabla f(\widetilde{\mathbf{w}}+\delta u)+\mathbf{E}_{u\sim \mathbb{B}}[\nabla f(\widetilde{\mathbf{w}}+\delta_ m u)].
\end{split}
\end{equation}

The lemma below states that $\v$ is an unbiased estimate of $\nabla\hat{F}_\delta (\mathbf{w})$.
\begin{lem}\label{lem:unbias}
Let $\mathbf{w}\in \mathbb{R}^d$, $\delta\geq 0$, then Eq.\eqref{eq:sample} is an unbiased estimate for $\nabla\hat{F}_\delta (\mathbf{w})$.
\end{lem}

Lemma \ref{lem:unbias} indicates that we can use Eq. \eqref{eq:sample} to estimate $\nabla\hat{F}_\delta (\mathbf{w})$. The following Theorem \ref{prop:variance_reduced} illustrates that the variance of $\v$ is reduced.
\begin{thm}\label{prop:variance_reduced}
$f$ is $L$-smoothness. For $k=1,2,\dots$, let
\begin{equation*}
\begin{split}
\v_k=\nabla h(\mathbf{w}_{k-1})+\nabla f(\mathbf{w}_{k-1}+\delta u)-\nabla f(\widetilde{\mathbf{w}}+\delta u)+\mathbf{E}_{u\sim \mathbb{B}}[\nabla f(\widetilde{\mathbf{w}}+\delta u)],
\end{split}
\end{equation*}
 then the variance of $\v_k$ is bounded. Moreover, if $\mathbf{w}_{k-1}$ and $\widetilde{\mathbf{w}}$ converge to optimum $\mathbf{w}_\ast$, then the variance of $\v_k$ also converges to zero.
\end{thm}

  Theorem \ref{prop:variance_reduced} implies that for the iterative formula $\mathbf{w}_k=\mathbf{w}_{k-1}-\eta \v_k$, step size $\eta$ can be set as a little larger constant. In contrast, the step size of Suffix-SGD decreases with the number of iteration increasing. So our method may converge faster than Suffix-SGD. Algorithm \ref{alg:a1} gives the full description of our method with constant step size $\eta$ for non-convex $(a,c,\sigma)$-nice functions.

 %The following convergence and complexity analysis depend on option \uppercase\expandafter{\romannumeral1}.
 \subsection{Convergence and Complexity Analyses}
 In this section, we will discuss the convergence and complexity analyses of  Algorithm \ref{alg:a1}. A point $\mathbf{w}$ is called $\varepsilon$-accurate solution, if $\mathbf{E}[F(\mathbf{w})]-F(\mathbf{w}_\ast)\leq\varepsilon$.
   First, we discuss the convergence and complexity of SVRG with project step. Because $\mathbf{w}_{k}$ is the projection of $(\mathbf{w}_{k-1}-\eta \v_k)$ onto $\mathcal{C}$, and $\mathbf{w}_\ast\in \mathcal{C}$, we have that $\|\mathbf{w}_{k}-\mathbf{w}_\ast\|^2\leq\|\mathbf{w}_{k-1}-\eta \v_k-\mathbf{w}_\ast\|^2$. According to the theorem in \cite{johnson2013accelerating} (in \cite{johnson2013accelerating}, let $\psi(\mathbf{w})=h(\mathbf{w})+f(\mathbf{w}+\delta u)$), we have the following theorem:
  \begin{thm}\label{thm:converge}
  Suppose $\hat F_\delta(\mathbf{w})$ is $L$-smoothness and $\sigma$-strong convex on $\mathcal{C}$, and let $\mathbf{w}_\ast=\arg\min_{\mathbf{w}} \hat F_\delta(\mathbf{w})$ be an optimal solution. In addition, assume that $T$ is sufficiently large so that
    \begin{equation}\label{eq:rho}
    \rho=\frac{1}{\sigma \eta(1-2L\eta)T}+\frac{2L\eta}{1-2L\eta}<1.
    \end{equation}
    Then the SVRG with project step in algorithm \ref{alg:a1} has geometric convergence in expectation:
    \begin{equation*}
    \mathbf{E} [\hat F_\delta(\widetilde{\mathbf{w}}_{s})]-\hat F_\delta(\mathbf{w}_\ast)\leq \rho^s[\hat F_\delta(\widetilde{\mathbf{w}}_{0})-\hat F_\delta(\mathbf{w}_\ast)]
    \end{equation*}
  \end{thm}
\begin{remark}\label{rem:converge}
In Eq. \eqref{eq:rho}, denote $\theta=\eta L$ and $T=100(L/\sigma)$. Then when $0<\theta<2/9$, we obtain that $\rho<1$. For example, let $\theta=0.2$, then $\rho=3/4<1$.
\end{remark}

\begin{cor}
In order to have an $\varepsilon$-accurate solution, the number of stages $s$ needs to satisfy
 \begin{equation*}
 s\geq\left(\log {\frac{1}{\rho}}\right)^{-1}\log\left({\frac{\hat F_\delta(\widetilde{\mathbf{w}}_0)-\hat F_\delta(\mathbf{w}_\ast)}{\varepsilon}}\right).
 \end{equation*}
\end{cor}
%Each stage of the SVRG requires $1+3T$ gradient calculations. Moreover, it is sufficient to set $T=\Theta(L/\sigma)$. So the overall complexity of SVRG is
%$$O((1+L/\sigma)\log(1/\varepsilon)).$$
Following Theorem gives the complexity of Algorithm \ref{alg:a1}:
\begin{thm}
Let $\mathcal{C}$ be a convex set, $M=\sqrt{{1}/{\varepsilon}}>1$, $0<c<1$, $\hat F(\cdot)$ be an $L$-smooth $(a,c,\sigma)$-nice function, then after $O(1/\varepsilon)$ optimization steps, Algorithm \ref{alg:a1} outputs a point $\mathbf{w}_{m+1}$ which is $\varepsilon$-accurate.
\end{thm}
\begin{proof}
Let $T_{total}$ be the total number of steps made by Algorithm \ref{alg:a1}, then we have:

  \begin{equation*}
  \begin{split}
  T_{total}&\leq \sum\limits_{m=1}^M \left(\log {\frac{1}{\rho}}\right)^{-1}\log{\frac{\hat F_\delta(\widetilde{\mathbf{w}}_0)-\hat F_\delta(\mathbf{w}_\ast)}{\varepsilon_m}}\\
   &=\sum\limits_{m=1}^M \left(\log {\frac{1}{\rho}}\right)^{-1}\log{\frac{\hat F_\delta(\widetilde{\mathbf{w}}_0)-\hat F_\delta(\mathbf{w}_\ast)}{\sigma c^2\delta_m^2/8}}\\
   &= \left(\log {\frac{1}{\rho}}\right)^{-1}\sum\limits_{m=1}^M\log{\frac{[\hat F_\delta(\widetilde{\mathbf{w}}_0)-\hat F_\delta(\mathbf{w}_\ast)]\cdot8}{\sigma\cdot(c^{m-1}\delta_1)^2}}\\
   &= \left(\log {\frac{1}{\rho}}\right)^{-1}[M\log\widetilde{\Gamma}-2\log c\sum\limits_{m=1}^M{(m-1)}]\\
  &= \left(\log {\frac{1}{\rho}}\right)^{-1}[M\log\widetilde{\Gamma}+(\log\frac{1}{c})(M^2-M)]\\
   &= \left(\log {\frac{1}{\rho}}\right)^{-1}[(\log\frac{1}{c})M^2+(\log{c\widetilde{\Gamma}})M]\\
  &\leq \left(\log {\frac{1}{\rho}}\right)^{-1}(\log{\widetilde{\Gamma}})M^2=\left(\log {\frac{1}{\rho}}\right)^{-1}(\log{\widetilde{\Gamma}})\frac{1}{\varepsilon}
  \end{split}
  \end{equation*}
  For simplicity, we set $\widetilde{\Gamma}=\frac{[\hat F_\delta(\widetilde{\mathbf{w}}_0)-\hat F_\delta(\mathbf{w}_\ast)]\cdot8}{\sigma \delta_1^2}$ in third equation. The last inequality holds because $M\leq M^2$ as $M\geq1$.
\end{proof}
 The complexity of GradOpt is $O(1/\varepsilon^2)$ \cite{hazan2016graduated}, whereas that of our SVRG-GOA is $O(1/\varepsilon)$. So our proposed algorithm has lower complexity than GradOpt.
\section{Graduated Optimization Algorithm with Prox-SVRG}
Proximal gradient method is another effective method to solve a sum of a convex function and a nonconvex function. Prox-SVRG was proposed by L. Xiao \cite{lin2014}. It has lower complexity comparing with proximal full gradient and proximal SGD. Based on Prox-SVRG, we present Algorithm \ref{alg:2}: PSVRG-GOA.

Comparing with SVRG-GOA in Section \ref{sec:svrg-goa}, PSVRG-GOA does not need computing the gradient of $h(\mathbf{w})$. In Algorithm \ref{alg:2}, the proximal mapping is defined as:
$$\text{prox}_h(\nu)=\mathop{\arg\min}_{\nu\in \mathbb{R}^d}\left\{\frac{1}{2}\|\mathbf{w}-\nu\|^2+h(\mathbf{w})\right\}.$$
If $h(\mathbf{w})=\frac{\lambda}{2}\|\mathbf{w}\|^2$, by simple computation, we obtain $\text{prox}_{\eta h}(\nu)=\frac{1}{1+\lambda \eta}\nu$.

The convergence analysis of Prox-SVRG from \cite{lin2014} and $\|\mathbf{w}_{k}-\mathbf{w}_\ast\|^2\leq\|\mathbf{w}_{k-1}-\eta \v_k-\mathbf{w}_\ast\|^2$ guarantees the convergence of the Prox-SVRG with project step in Algorithm \ref{alg:2}:

 \begin{algorithm}[htp]
\caption{\textbf{\emph{PSVRG-GOA}} $-$ \emph{\textbf{G}raduated \textbf{O}ptimization \textbf{A}lgorithm with \textbf{P}roximal \textbf{SVRG}}}\label{alg:2}%算法的名字
\hspace*{0.02in} {\bf Input:} %算法的输入， \hspace*{0.02in}用来控制位置，同时利用 \\ 进行换行
target error $\varepsilon>0$, step size $0<\eta<1$, decision set $\mathcal{C}$, iterative number $T$, shrink factor $0<c<1$ and $M=\sqrt{1/\varepsilon}\geq1$. Choose initial point $\mathbf{w}_1\in\mathcal{C}$ uniformly at random. Set $\delta_1=\text{diam}(\mathcal{C})$.\\
\hspace*{0.02in} {\bf Output:} %算法的结果输出
$\mathbf{w}_{m+1}$
\begin{algorithmic}[1]
%\STATE Set $\varepsilon_m:=\sigma c^2\delta_m^2/8$, and
%     $$S=\left(\log{\frac{\sigma \eta(1-2L\eta)T}{1+2L\eta^2\sigma T}}\right)^{-1}\log{\frac{F(\mathbf{w}_m)-F(\mathbf{w}_\ast)}{\varepsilon_m}}$$\label{cod:s}
%
%    \STATE Set shrunk decision set $$\mathcal{C}_m:=\mathcal{C}\cap\mathbb{B}(\mathbf{w}_m,1.5\delta_m)$$
%
%    \STATE $\widetilde{\mathbf{w}}_0=\mathbf{w}_m$//Perform variant of SVRG over $\hat f_{\delta_ m}$
\FOR{$m=1$ to $M$}
    \STATE Set $\varepsilon_m:=\frac{\sigma c^2\delta_m^2}8$, $S=\left(\log{\frac{\sigma \eta(1-4L\eta)T}{1+4L\sigma \eta^2(T+1)}}\right)^{-1}\log{\frac{\Delta F_m}{\varepsilon_m}},$ with $\Delta F_m=F(\mathbf{w}_m)-F(\mathbf{w}_\ast)$;
    \STATE Set shrunk decision set $\mathcal{C}_m:=\mathcal{C}\cap\mathbb{B}(\mathbf{w}_m,1.5\delta_m)$;
    \STATE $\widetilde{\mathbf{w}}_0=\mathbf{w}_m$; \texttt{//Perform Prox-SVRG over $\hat f_{\delta_ m}$}
    \FOR{$s=1$ to $S$}  \label{cod:begin_SVRG2}
           \STATE $\widetilde{\mathbf{w}}=\widetilde{\mathbf{w}}_{s-1}$
           \STATE $\widetilde{g}=\nabla f(\widetilde{\mathbf{w}})$\label{cod:full_gradient2}
           \STATE $\mathbf{w}_0=\widetilde{\mathbf{w}}$
            \FOR {$k=1$ to $T$ }
               \STATE    Randomly pick $u\in\mathcal{C}_m$ and update
               \STATE    $\v_k=\nabla f(\mathbf{w}_{k-1}+\delta_ m u)-\nabla f(\widetilde{\mathbf{w}}+\delta_ m u)+\widetilde{g}$\label{cod:v_k2}
               \STATE    $\mathbf{w}_k=\Pi_{\mathcal{C}_m}[\text{prox}_{\eta h}(\mathbf{w}_{k-1}-\eta \v_k)]$ \label{cod:prox2}
            \ENDFOR
            \STATE  set $\widetilde{\mathbf{w}}_{s}=\mathbf{w}_k$ for randomly chosen $k\in\{0,\dots,T-1\}$\label{cod:option12}
           % \State \textbf{option \uppercase\expandafter{\romannumeral2}}: set $\widetilde{\mathbf{w}}_{s}=\frac{1}{T}\sum_{k=1}^{T}\mathbf{w}_k$ \label{cod:option22}
      \ENDFOR\label{cod:finish_SVRG2}
      \STATE $\mathbf{w}_{m+1}=\widetilde{\mathbf{w}}_{s}$
      \STATE $\delta_{m+1}=c\delta_{m}$
\ENDFOR
%\State \Return $\mathbf{w}_{m+1}$
\end{algorithmic}
\end{algorithm}

\begin{thm}\label{thm:converge}
  Suppose $\hat F_\delta(\mathbf{w})$ is $L$-smooth and $\sigma$-strong convex on $\mathcal{C}$, and let $\mathbf{w}_\ast=\arg\min_{\mathbf{w}} \hat F_\delta(\mathbf{w})$ be an optimal solution. In addition, assume that $T$ is sufficiently large so that
    \begin{equation}
    \rho=\frac{1}{\sigma \eta(1-4L\eta)T}+\frac{4L\eta(T+1)}{(1-4L\eta)T}<1.
    \end{equation}
    Then the Prox-SVRG with project step in algorithm \ref{alg:a1} has geometric convergence in expectation:
    \begin{equation}
    \mathbf{E} [\hat F_\delta(\widetilde{\mathbf{w}}_{s})]-\hat F_\delta(\mathbf{w}_\ast)\leq \rho^s[\hat F_\delta(\widetilde{\mathbf{w}}_{0})-\hat F_\delta(\mathbf{w}_\ast)]
    \end{equation}
  \end{thm}
  Following theorem gives the complexity of Algorithm \ref{alg:2}:
\begin{thm}
Let $\mathcal{C}$ is a convex set, $\hat F$ be an $L$-smooth $(a,c,\sigma)$-nice function, then after $O(\log(1/\varepsilon))$ optimization steps, Algorithm 2 outputs a point $\mathbf{w}_{m+1}$ which is $\varepsilon$-accurate.
\end{thm}
\section{Extensions}
In this section, we give two extensions for Algorithm \ref{alg:a1} and \ref{alg:2}.
%One extents our methods to solving problem \eqref{eq:problem} with $f(w)$ as the limit-sum of some non-convex functions, which is often met in machine learning. The other one uses mini-batch strategy to increase the parallelism and to further reduce the variance of the stochastic gradient method.

\textbf{Limit-Sum of Non-convex Function:}
 In machine learning, we often encounter the problem $$F(\mathbf{w})=h(\mathbf{w})+\frac{1}{n}\sum\limits_{i=1}^n f_i(\mathbf{w}),$$ where $f_i(\mathbf{w})$ is non-convex for $i=1,\cdots,n$. To deal with such optimization problem, we need to replace $\nabla f(\cdot)$  with $\frac{1}{n}\sum\limits_{i=1}^n\nabla f_{i}(\cdot)$ in Step \ref{cod:v_k} in our Algorithms. It is no doubt that the calculation is mass. To simplify the calculation, we randomly pick $i$ from $\{1,\ldots, n\}$, and then calculate $\nabla f_i(\cdot)$ instead of $\nabla f(\cdot)$ in Step \ref{cod:v_k}. This method is reasonable and practical since $\mathbf{E}[\nabla f_{i}(\cdot)]=\nabla f(\cdot)$, and can greatly reduce the computation cost.

\textbf{Mini-batch:}
In this section, we give an extension for Algorithm \ref{alg:a1} and \ref{alg:2} using mini-batch strategy.
 Mini-batching is a popular strategy in distributed and multicore setting since it helps to increase the parallelism and to further reduce the variance of the stochastic gradient method. For simplicity, we only give the key differences between the mini-batch SVRG-GOA/PSVRG-GOA and SVRG-GOA/PSVRG-GOA. To apply mini-batch strategy, we replace line \ref{cod:random} and \ref{cod:v_k} in Algorithms \ref{alg:a1} and \ref{alg:2} with the following updates:
 \begin{enumerate}
 \item Randomly pick $U\in \mathcal{C}_m$ such that $|U|=b$, and update
 \item $\v_k=\nabla h(\mathbf{w}_{k-1})+\frac{1}{b}\sum_{u\in U}[\nabla f(\mathbf{w}_{k-1}+\delta_ m u)-\nabla f(\widetilde{\mathbf{w}}+\delta_ m u)+\widetilde{g}]$ (for SVRG-GOA).
 \item $\v_k=\frac{1}{b}\sum_{u\in U}[\nabla f(\mathbf{w}_{k-1}+\delta_ m u)-\nabla f(\widetilde{\mathbf{w}}+\delta_ m u)]+\widetilde{g}$ (for PSVRG-GOA).
 \end{enumerate}
 When $b=1$, mini-batch SVRG-GOA and mini-batch PSVRG-GOA are changed into Algorithm \ref{alg:a1} and \ref{alg:2}. Mini-batch strategy can increase the parallelism of the algorithms. In our experiments, we main discuss serial computing performance of the proposed algorithms.
\section{Experiments}\label{sec:experiments}
To illustrate the influence factors of the new methods, and to compare the performances of SVRG-GOA and PSVRG-GOA with several related algorithms, we present some results of numerical experiments. In every figure, $x$-axis is the number of effective passes over the data, where each effective pass performs one outer loop of the algorithms, and the total number of inner loops for all the compared algorithms are set the same. Each experiment is repeated many times independently, and the data in figures are the average results.
%The $y$-axis is the function suboptimality i.e. $F(\mathbf{w}_k)-F(\mathbf{w}_\ast)$, here $\mathbf{w}_\ast$ refers to the best solution obtained by running four algorithms, i.e. SVRG-GOA, PSVRG-GOA, GradOpt and nonconvex Prox-SVRG, for a large number of iterations and with multiple random starting points.

We focus on the differentiable robust least squares support vector machine \cite{chen2017} for binary classification: given a set of training examples $(\mathbf{x}_1,y_1),\dots,(\mathbf{x}_n,y_n)$, where $\mathbf{x}_i\in\mathbb{R}^d$ and $y_i\in\{+1,-1\}$, we find the optimal predictor $\mathbf{w}\in\mathbb{R}^d$ by solving
\begin{equation*}
\min\limits_{\mathbf{w}\in\mathbb{R}^d} \frac{\lambda}{2}\|\mathbf{w}\|^2+\frac{1}{n}\sum\limits_{i=1}^nL(\xi_i),
\end{equation*}
where
\begin{equation*}
\begin{split}
L\left(\xi_i\right)= \frac{1}{2}[\xi_i^2-\max\{0,\xi_i^2-\tau^2\}]-\frac{1}{2p}\log(1+\exp(-p|\xi_i^2-\tau^2|)),
\end{split}
\end{equation*}
$\tau$ is the truncation parameter, and $\xi_i=y_i-\mathbf{w}^\top \mathbf{x}_i$.
%Figure \ref{fig:rlssvm_1dim} gives the plot of SRLS-SVM in 1-dimensional space. The positive and negative samples are subject to $N(0,1)+1.5$ and $N(0,1)-1.5$ respectively. $\mathbf{w}$ are uniformly drawn from interval $[-3,3]$. The parameters are set as $\lambda=10^{-3}$ and $\tau=0.1$.
%for simplicity, that is, we use the linear kernel ($\varphi(\mathbf{x}_i)=\mathbf{x}_i$) and omit the bias $b$ for the ordinary version $\mathbf{w}^\top \varphi(\mathbf{x}_i)+b$. SRLS-SVM model \eqref{eq:rlssvm} is a non-convex function, because its loss function is non-convex.

In terms of the form of model \eqref{eq:problem} with $f(\mathbf{w})=\frac{1}{n}\sum_{i=1}^n f_i(\mathbf{w})$, we have
$$h(\mathbf{w})=\frac{\lambda}{2}\|\mathbf{w}\|^2,~~~~~f_i(\mathbf{w})=L(\xi_i).$$
%\makeatletter\def\@captype{figure}\makeatother
%\begin{minipage}[b]{.48\linewidth}
%\centering
%\includegraphics[width=\textwidth]{rlssvm_1dim.pdf}
%\caption{The plot of nonconvex function SRLS-SVM in 1-dimensional space. $x$-axis denotes $\xi=y-w^\top x$. we choose $\lambda=10^{-3}$ and $\tau=0.1$.}\label{fig:rlssvm_1dim}
%\end{minipage}%
%\quad
%\makeatletter\def\@captype{table}\makeatother
%\begin{minipage}[b]{.48\linewidth}
%\centering
%\begin{tabular}{cc}
%\hline
%$N_{total}$&$N_{global}$\\
%\hline
%100&71\\
%%2&100&22\\
%%3&100&24\\
%300&241\\
%500&406\\
%700&554\\
%1000&811\\
%\hline
%\end{tabular}
%\caption{Statistics of the number of converging to global optimums for nonconvex prox-SVRG method on breast cancer data set. $N_{total}$ denotes the number of experiments, and $N_{global}$ denotes the number of converging to global optimums}
%\label{tab:2}
%\end{minipage}
\subsection{Influence Factors of SVRG-GOA}
In this section, we discuss some factors which influence the performance of the proposed algorithms. Experimental results in Section \ref{sec:experiment_2} are shown that PSVRG-GOA and SVRG-GOA always have the same performances, so we only discuss influence factors of SVRG-GOA.
\subsubsection{Step Sizes $\eta$}
Fig. \ref{fig:bc_new_eta} shows the performance of SVRG-GOA with different step sizes $\eta$. It can be seen from the Fig. \ref{fig:bc_new_eta} that the convergence speed of SVRG-GOA becomes slow if $\eta$ is set too small. Furthermore, when $\eta$ is set too large, SVRG-GOA may converge to a local optimum. So step size cannot be set too large or too small. In this experiment, $\eta=0.2$ is the best choice. In practice, $\eta$ can be chosen according to Remark \ref{rem:converge}, that is $0<\eta<\frac{2}{9L}$.
\begin{figure}[htp]
%\begin{minipage}[t]{0.5\linewidth}
\centering
\includegraphics[width=0.7\textwidth]{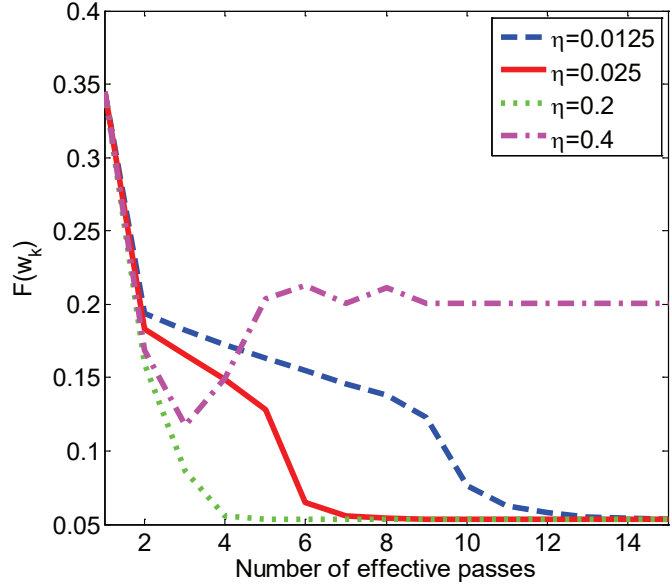}
\caption{SVRG-GOA on breast cancer data set: varying step size $\eta$. The best $\eta\in(0, \frac{2}{9L})$. }\label{fig:bc_new_eta}
%\end{minipage}
%\quad
%\begin{minipage}[t]{0.5\linewidth}
\end{figure}
\subsubsection{Initial smoothing factor $\delta$}
We vary the initial value of smoothing factor $\delta$ for GradOpt and SVRG-GOA on breast cancer data set. Fig. \ref{fig:delta1} illustrates how $F(\mathbf{w}_k)$ decreases as the increasing of the number of effective passes. In general, SVRG-GOA converges faster than GradOpt for any choices of $\delta$. Furthermore, if $\delta$ is set too large, algorithms convergence slowly for both GradOpt and SVRG-GOA. GradOpt may even not converge to the global optimum when $\delta$ is set too small. In this experiment, $\delta=1$ is a good choice. In practice, $\delta$ can be chosen by cross-validation technique.

\begin{figure}[htp]
\centering
\includegraphics[width=0.7\textwidth]{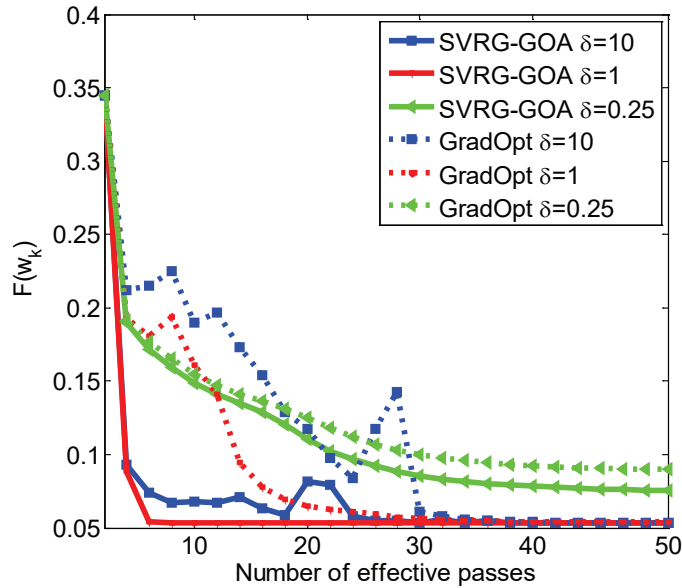}
\caption{Comparison of GradOpt and SVRG-GOA on breast cancer data set with $\eta=0.1$ : vary the initial value of $\delta$.}
\label{fig:delta1}
%\end{minipage}
\end{figure}
%\begin{figure}
% \centering
% \subfloat[GradOpt]{\label{fig:bc_old_delta}\includegraphics[width=0.48\textwidth]{breast_cancer_old_delter_error_init_bc.pdf}}
% \subfloat[SVRG-GOA]{\label{fig:bc_new_delta}\includegraphics[width=0.48\textwidth]{breast_cancer_new_delter_error_init_bc.pdf}}
%\caption{Comparison of GradOpt and SVRG-GOA on breast cancer data set with $\eta=0.1$ : vary the initial value of $\delta$. Here $\mathbf{w}_\ast$ refers to the best solution obtained by running SVRG-GOA for a large number of iterations and with multiple random starting points.}
%\label{fig:delta}
%\end{figure}
%\begin{figure}
% \centering
% \includegraphics[width=0.48\textwidth]{breast_cancer_delter_error_init_bc.pdf}
%\caption{Comparison of GradOpt and SVRG-GOA on breast cancer data set with $\eta=0.1$ : vary the initial value of $\delta$. Here $\mathbf{w}_\ast$ refers to the best solution obtained by running SVRG-GOA for a large number of iterations and with multiple random starting points.}
%\label{fig:delta1}
%\end{figure}
\subsection{Comparison with Related Algorithms}\label{sec:experiment_2}
In order to illustrate the performances and properties of our methods, we compare the following algorithms:
\begin{itemize}
\item Nonconvex proxSVRG: the nonconvex proximal SVRG given in \cite{reddi2016proximal} with batch size $b=1$. The step size of it is a constant. This algorithm is denoted as Prox-SVRG in figures for simplicity.
\item GradOpt: graduated optimization with a gradient oracle in \cite{hazan2016graduated}. It uses Suffix-SGD \cite{rakhlin2011making} in the algorithm, and the step size ($\eta_k=1/\sigma k$) is reduced with the increasing of the iteration number $k$.
\item SVRG-GOA: the method proposed in Algorithm \ref{alg:a1} in this paper.
\item PSVRG-GOA: the method proposed in Algorithm 2 in this paper using proximal SVRG.
\end{itemize}

We use publicly available data sets. Their sizes $n$, dimensions $d$ as well as sources are listed in Table \ref{tab:1}. Table \ref{tab:1} also lists the values of $\lambda$ and $\tau$ that were used in our experiments. These choices are typical in machine learning benchmarks to obtain good performance.
\begin{table}[htp]
\centering
\caption{Detailed information of data sets, regularization parameter $\lambda$ and truncation parameter $\tau$ used in our experiments}
\begin{tabular}{cccccc}
\hline
Data Sets&$n$&$d$&Source&$\lambda$&$\tau$\\
\hline
breast cancer&683&10&\cite{lib}&$10^{-3}$&0.9\\
\hline
covtype&581,012&54&\cite{lib}&$10^{-6}$&0.9\\
\hline
sido0&12,678&4,932&\cite{Guyon2008}&$10^{-3}$&0.9\\
\hline
svmguide1&7,089&4&\cite{lib}&$10^{-3}$&0.9\\
\hline
IJCNN1&141,691&22&\cite{lib}&$10^{-3}$&2.5\\
\hline
adult&48,842&123&\cite{lib}&$10^{-5}$&1.5\\
\hline
\label{tab:1}
\end{tabular}
\end{table}

\begin{figure}[htp]
\centering
\subfloat[Breast cancer]{\label{fig:bc_error}\includegraphics[width=0.32\textwidth]{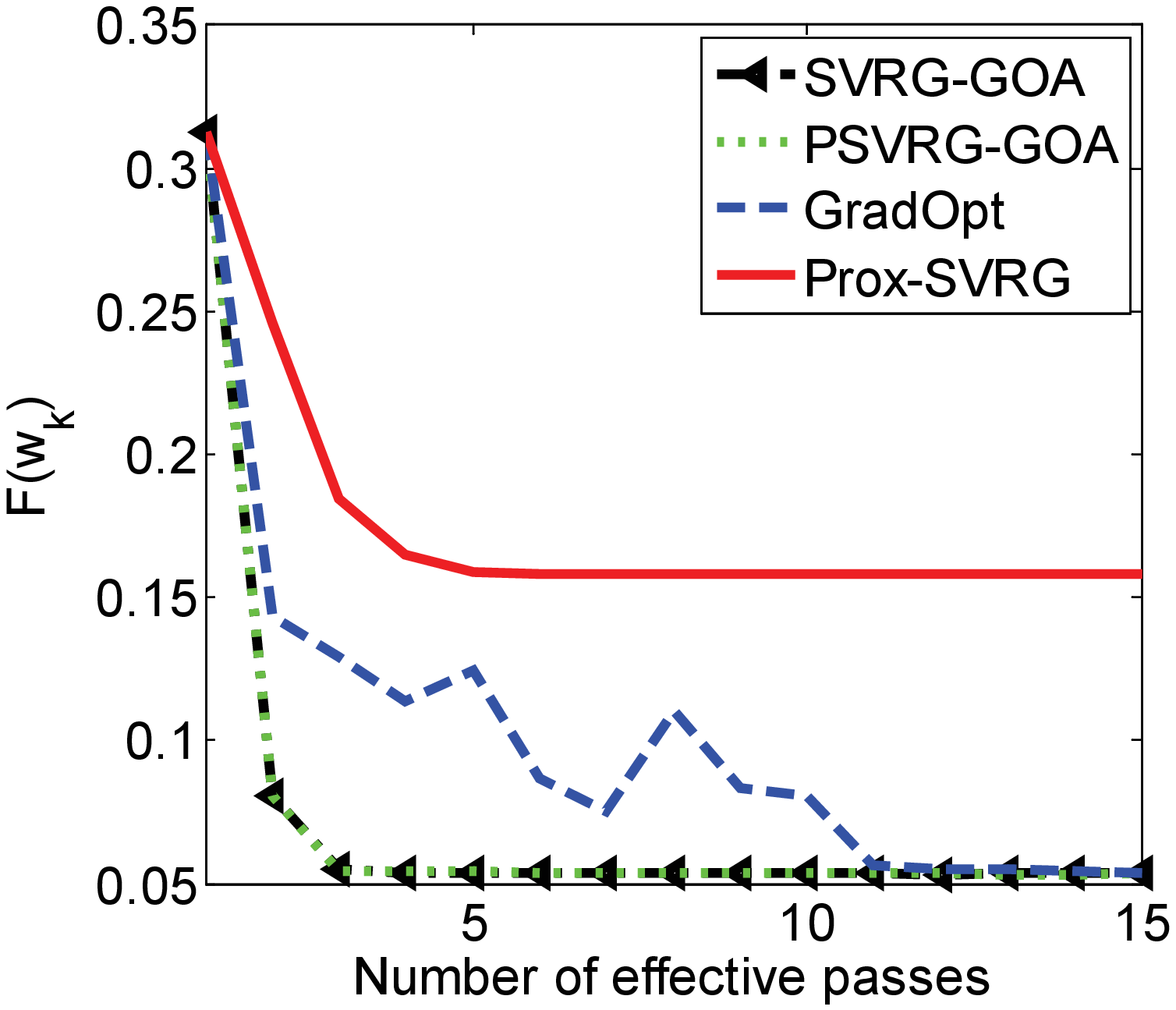}}
\subfloat[Covtype]{\label{fig:cov_error}\includegraphics[width=0.32\textwidth]{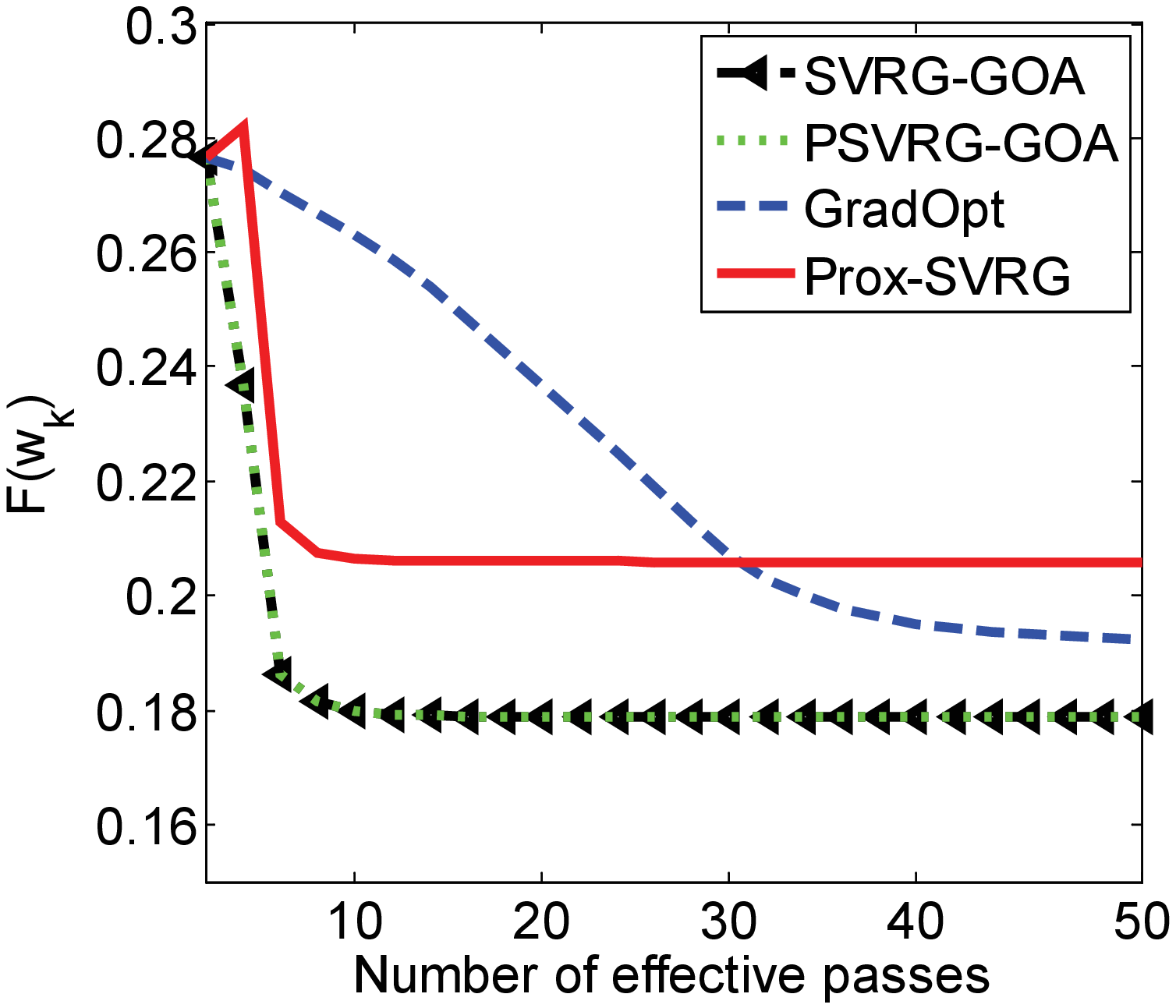}}
\subfloat[Sido0]{\label{fig:sido_error}\includegraphics[width=0.32\textwidth]{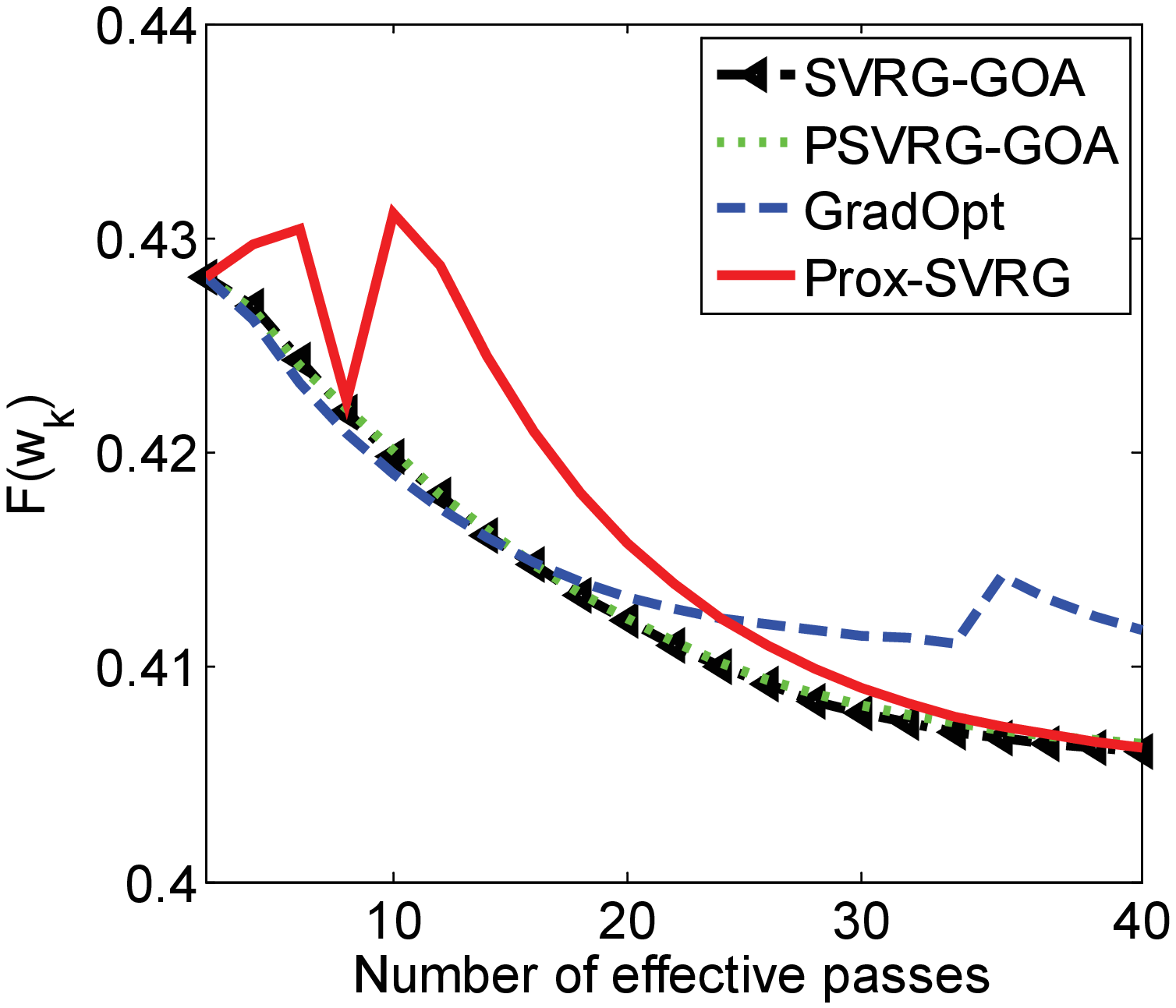}}\\
\subfloat[SVMguide1]{\label{fig:svmguide1_error}\includegraphics[width=0.32\textwidth]{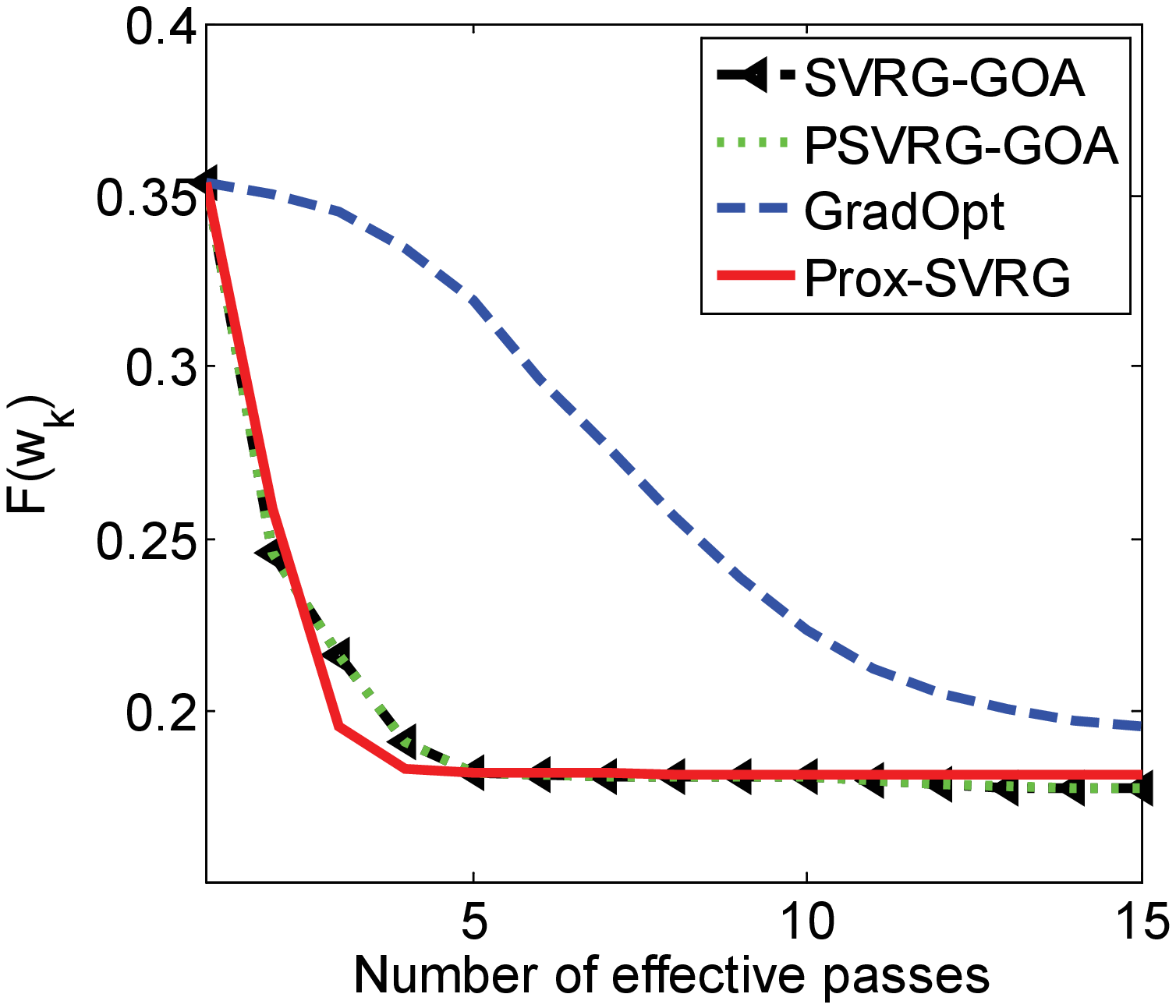}}
\subfloat[IJCNN1]{\label{fig:IJCNN1_error}\includegraphics[width=0.32\textwidth]{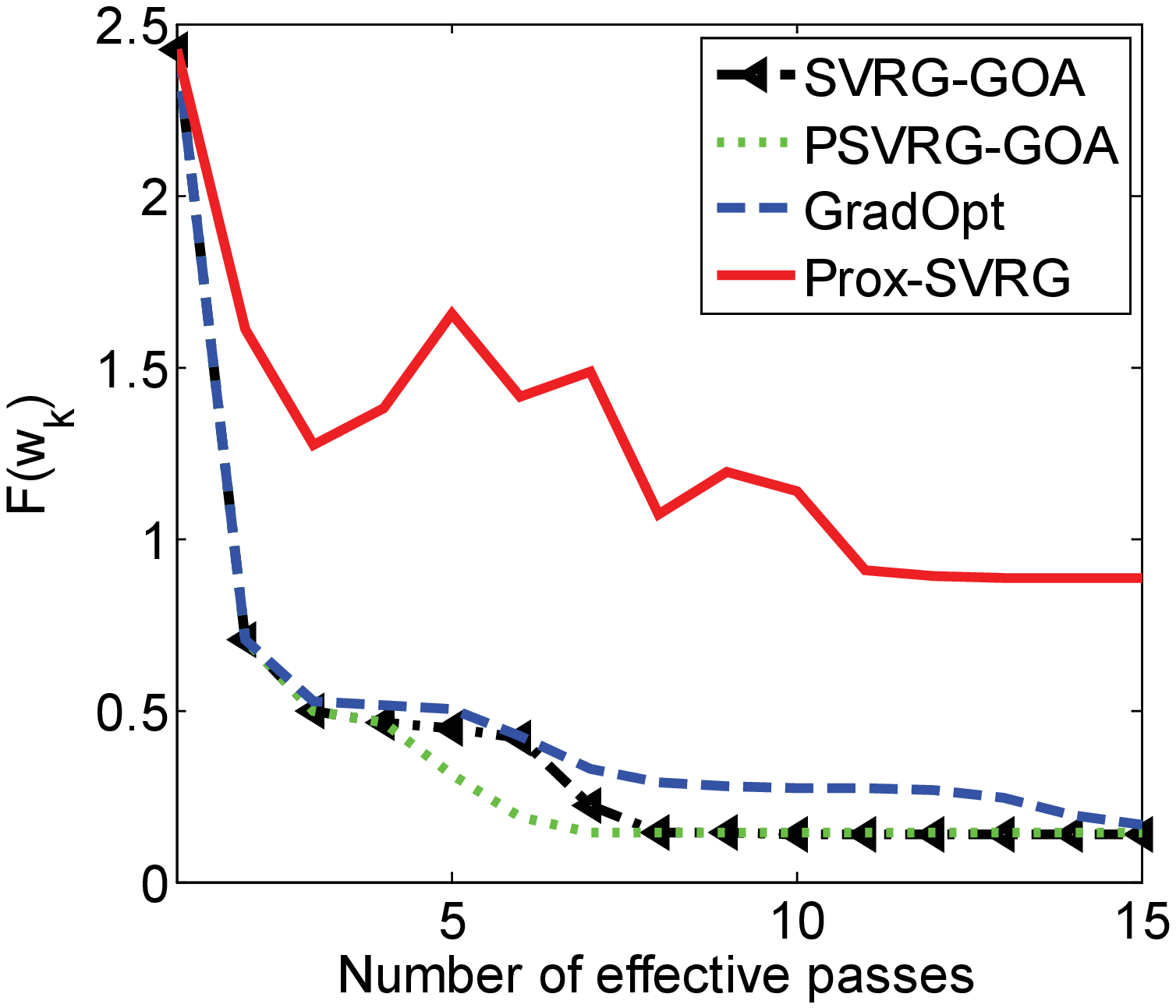}}
\subfloat[Adult]{\label{fig:adult_error}\includegraphics[width=0.32\textwidth]{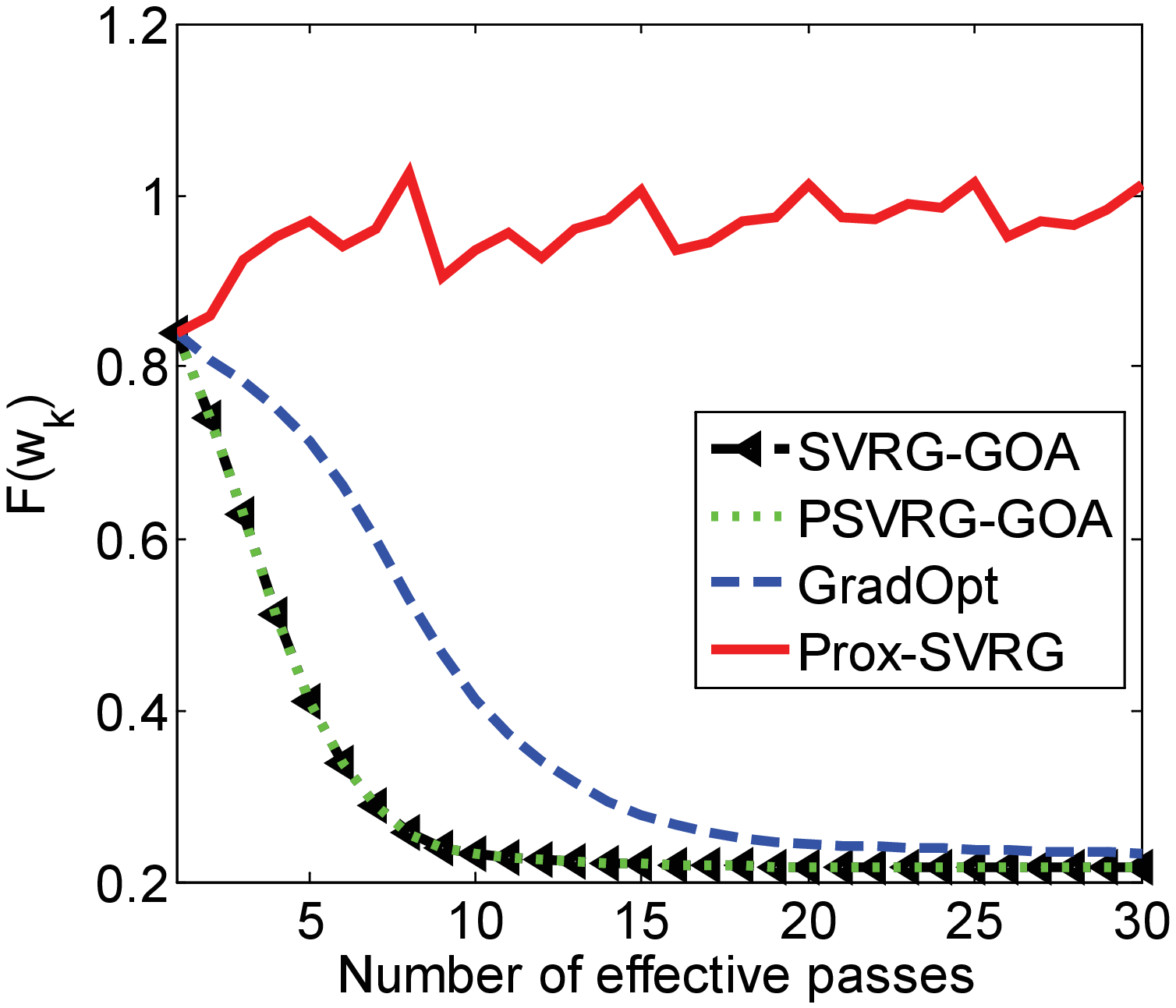}}
\caption{Comparison of different methods. Prox-SVRG is the nonconvex ProxSVRG method with batch size $b=1$.}
\label{fig:error}
\end{figure}

Fig. \ref{fig:error} shows the comparison of different methods on the data sets listed in Table \ref{tab:1}. We set the initial approximate parameter $\delta=1$, which decayed by shrinkage factor $c=0.9$ for SVRG-GOA, PSVRG-GOA and GradOpt algorithms. Step size was set as $\eta=0.2$, which matches our theoretical analysis (see the Remark \ref{rem:converge}) for SVRG-GOA, PSVRG-GOA, and nonconvex proxSVRG. It can be seen that SVRG-GOA and PSVRG-GOA are both able to converge to the optimum and the convergence speed is faster than other methods. These two methods always have the similar performances. The superior performances of SVRG-GOA and PSVRG-GOA are due to their low complexity. In most cases, GradOpt may converge to the optimum, but the convergence speed is slower than SVRG-GOA and PSVRG-GOA.
Nonconvex ProxSVRG is a convergent algorithm for nonconvex optimization problems. But it sometimes converges to the global optimums, and sometimes not, because nonconvex ProxSVRG is a stochastic method. We did a large number of experiments for nonconvex ProxSVRG on breast cancer data set, and then counted the number of converging to global optimums. The probability of converging to the global optimum is approximately 80.12\% in this experiment.
 %Table \ref{tab:2} statistics the number of converging to global optimums for nonconvex ProxSVRG on breast cancer data set. A short calculation reveals that the probability of converging to global optimum for nonconvex ProxSVRG on breast cancer data set is approximately 80.12\%.
 In contrast, the performances of SVRG-GOA, PSVRG-GOA, and GradOpt are stable because, in each iteration, these three methods deal with a local strong convex optimization problem whose optimum is unique in $\mathcal{C}_m$.
%\begin{table}
%\centering
%\caption{Statistics of the number of converging to global optimums for nonconvex prox-SVRG method on breast cancer data set. $N_{total}$ denotes the number of experiments, and $N_{global}$ denotes the number of converging to global optimums. The probability of converging to global optimums for nonconvex prox-SVRG on breast cancer data set is approximately 80.12\%}
%\begin{tabular}{cc}
%\hline
%$N_{total}$&$N_{global}$\\
%\hline
%100&71\\
%%2&100&22\\
%%3&100&24\\
%300&241\\
%500&406\\
%700&554\\
%1000&811\\
%\hline
%\end{tabular}
%\label{tab:2}
%\end{table}
\section{Discussion}
In this paper, we present two algorithms based on graduate optimization to obtain the global optimum of a family of non-convex optimization. In the new algorithms, (proximal) stochastic variance reduction gradient technique with project step, new convex approximate and a new shrinkage factor are applied. We prove that our algorithms have lower complexity than GradOpt. Furthermore, we extend our analysis to mini-batch variants and to solving non-convex finite-sum problem. Experimental results show that our methods perform better than GradOpt and nonconvex Prox-SVRG.

However, there are also some interesting questions that remain to study:
\begin{itemize}
\item How to obtain the global optimum of other non-convex optimization problems? Can the analysis in this paper extend to other non-convex problems besides the $(a,c,\sigma)$-nice functions?
\item The convergence rates of our methods are both $O(1/\varepsilon)$. Are there second-order or other methods which can further accelerate convergence rate?
\end{itemize}
\subsection*{Acknowledgments}
We would like to acknowledge the support of National Natural Science Foundation of China(NNSFC) under Grant No.71301067; and the Fundamental Research Funds for the Central Universities under Grant No. JB150718.
%\section*{References}
\bibliographystyle{abbrv}
\bibliography{SVRG}
\begin{appendix}
\section{Proof of Lemma \ref{lem:unbias}}
\begin{proof}
By the linearity of expectation, we have
   \begin{equation*}\label{eq:unbias}
   \begin{split}
\mathbf{E}_{u\sim \mathbb{B}}(v)&=\mathbf{E}_{u\sim \mathbb{B}}[\nabla h(\w)+\nabla f(\w+\delta u)-\nabla f(\widetilde{\w}+\delta u)+\mathbf{E}_{u\sim \mathbb{B}}[\nabla f(\widetilde{\w}+\delta u)]\\
&=\nabla h(\w)+\mathbf{E}_{u\sim \mathbb{B}}[\nabla f(\w+\delta u)]-\mathbf{E}_{u\sim \mathbb{B}}\left[\nabla f(\widetilde{\w}+\delta u)+\mathbf{E}_{u\sim \mathbb{B}}\left[\nabla f(\widetilde{\w}+\delta u)\right]\right]\\
&=\nabla h(\w)+\mathbf{E}_{u\sim \mathbb{B}}[\nabla f(\w+\delta u)]=\nabla\hat{F}_\delta (\w).
   \end{split}\end{equation*}
\end{proof}
\section{Proof of Theorem \ref{prop:variance_reduced}}
\begin{proof}
First, we give a lemma which will be used in the following proof.
\begin{lem}\label{lem:bound}
Let $w_\ast$ be the optimum point of $\hat F_\delta(w)=h(w)+\hat f_\delta(w)$, $h(w)$ is convex and $f$ is $L$-smoothness, $u\sim\mathbb{B}$. We have
\begin{equation*}
\mathbf{E}\|\nabla f(w+\delta u)-\nabla f(w_\ast+\delta u)\|^2\leq 2L[\hat F_\delta(w)-\hat F_\delta(w_\ast)].
\end{equation*}
\end{lem}

Now we prove the Theorem.
   \begin{eqnarray*}
   %\begin{split}
&&\mathbf{E}_{u\sim \mathbb{B}}\|\v_k-\mathbf{E}_{u\sim \mathbb{B}}(\v_k)\|^2\\
   &=&\mathbf{E}_{u\sim \mathbb{B}}\|\nabla f(\w_{k-1}+\delta u)-\nabla f(\widetilde{\w}+\delta u)+\mathbf{E}_{u\sim \mathbb{B}}[\nabla f(\widetilde{\w}+\delta u)-\nabla f(\w_{k-1}+\delta u)]\|^2\\
   &=&\mathbf{E}_{u\sim \mathbb{B}}\|\nabla f(\w_{k-1}+\delta u)-\nabla f(\widetilde{\w}+\delta u)\|^2-\|\mathbf{E}_{u\sim \mathbb{B}}[\nabla f(\widetilde{\w}+\delta u)-\nabla f(\w_{k-1}+\delta u)]\|^2\\
   &\leq&\mathbf{E}_{u\sim \mathbb{B}}\|\nabla f(\w_{k-1}+\delta u)-\nabla f(\widetilde{\w}+\delta u)\|^2\\
   &=&\mathbf{E}_{u\sim \mathbb{B}}\|\nabla f(\w_{k-1}+\delta u)-\nabla f({\w_\ast}+\delta u)+\nabla f({\w_\ast}+\delta u)-\nabla f(\widetilde{\w}+\delta u)\|^2\\
   &\leq&\mathbf{E}_{u\sim \mathbb{B}}[2\|\nabla f(\w_{k-1}+\delta u)-\nabla f({\w_\ast}+\delta u)\|^2+2\|\nabla f({\w_\ast}+\delta u)-\nabla f(\widetilde{\w}+\delta u)\|^2]\\
   &\leq &4L[\hat f_\delta(\w_{k-1})-\hat f_\delta(\w_{\ast})+\hat f_\delta(\widetilde{\w})-\hat f_\delta(\w_{\ast})]\\
   &\leq &4L[\hat F_\delta(\w_{k-1})-\hat F_\delta(\w_{\ast})+\hat F_\delta(\widetilde{\w})-\hat F_\delta(\w_{\ast})]
   %\end{split}
   \end{eqnarray*}
The fourth equality uses $\mathbf{E}\|\xi-\mathbf{E}\xi\|^2=\mathbf{E}\|\xi\|^2-\|\mathbf{E}\xi\|^2\leq \mathbf{E}\|\xi\|^2$ for any random vector $\xi$. The second inequality uses $\|a+b\|^2\leq2\|a\|^2+2\|b\|^2$. The last inequality uses the Lemma \ref{lem:bound} twice.

Therefore, when both $\w_{k-1}$ and $\widetilde{\w}$ converge to $\w_\ast$, then the variance of $\v_k$ also converge to zero.
\end{proof}
\section{Proof of Lemma \ref{thm:6.1}}\label{pro:lem6.1}
\begin{proof}
We prove this Lemma by induction. Firstly, let us prove that the lemma holds for $m=1$. Note that $\delta_1=\text{diam}(\mathcal{C})$, therefore $\mathcal{C}_1=\mathcal{C}$, and also $\w_1^\ast\in\mathcal{C}_1$. Recall that $(a,c,\sigma)$-nice of $f$ implies that $\hat{\delta_1}$ is $\sigma$-strong convex in $\mathcal{C}$. Secondly, assume that lemma holds for $m>1$. By this assumption, $\hat f_{\delta_m}$ is $\sigma$-strong convex in $\mathcal{C}_m$, and also $\w_m^\ast\in \mathcal{C}_m$. The $\sigma$-strong convexity in $\mathcal{C}_m$ implies,
\begin{equation*}
\begin{split}
\|\w_{m+1}-\w_{m+1}^\ast\|\leq \sqrt{\frac{2}{\sigma}}\sqrt{\hat F_{\delta_m}(\w_{m+1})-\hat F_{\delta_m}(\w_{m}^\ast)}\leq\frac{\delta_{m+1}}{2}
\end{split}
\end{equation*}
Combining the latter with the property of $(a,c,\sigma)$-nice functions yields:
\begin{equation*}
\|\w_{m+1}-\w_{m+1}^\ast\|\leq \|\w_{m+1}-\w_m^\ast\|+\|\w_m^\ast-\w_{m+1}^\ast\|\leq\frac{\delta_{m+1}}{2}+\delta_{m+1}\leq 1.5\delta_{m+1}
\end{equation*}
and it follows that,
$$\w_{m+1}^\ast\in\mathbb{B}(\w_{m+1},1.5\delta_{m+1})\subset \mathbb{B}(\w_{m+1},r\delta_{m+1}).$$
Recalling that $\mathcal{C}_{m+1}:=\mathbb{B}(\w_{m+1},1.5\delta_{m+1})$, and the local strong convexity of $F$, then the induction step of the lemma holds.
\end{proof}
\section{Proof of Lemma \ref{lem:bound}}
\begin{proof}
First, we give a lemma which will be used in the following proof.
\begin{lem}\label{lem:smooth}
Assume $f$ is $L$-smoothness, then $$\frac{1}{2L}\|\nabla f(\w)\|^2\leq f(\w)-\min_\alpha f(\alpha).$$
\end{lem}

Consider the function
$$ \varphi(\w)=f(\w+\delta u)-f(\w_\ast+\delta u)-\nabla f(\w_\ast+\delta u)^\top (\w-\w_\ast),$$
then $\nabla \varphi(\w)=\nabla f(\w+\delta u)-\nabla f(\w_\ast+\delta u)$, and $\nabla \varphi(\w_\ast)=0$. Hence $\min_\w\varphi(\w)=\varphi(\w_\ast)=0$. Since $\varphi(\w)$ is $L$-smoothness, by Lemma \ref{lem:smooth}, we have
\begin{equation*}
\begin{split}
\frac{1}{2L}\|\nabla \varphi(\w)\|^2\leq\varphi(\w)-\min_v\varphi(v)=\varphi(\w)-\varphi(\w_\ast)=\varphi(\w).
\end{split}
\end{equation*}
This implies
\begin{equation*}
\begin{split}
\|\nabla f(\w+\delta u)-\nabla f(\w_\ast+\delta u)\|^2
\leq2L[f(\w+\delta u)-f(\w_\ast+\delta u)-\nabla f(\w_\ast+\delta u)^\top (\w-\w_\ast)].
\end{split}
\end{equation*}
By taking expectation on both sides of the above inequality, we obtain
\begin{equation}\label{eq:bound_1}
\begin{split}
\mathbf{E}\|\nabla f(\w+\delta u)-\nabla f(\w_\ast+\delta u)\|^2
\leq 2L[\hat{f}_\delta(\w)-\hat{f}_\delta(\w_\ast)-\nabla\hat{f}_\delta(\w_\ast)(\w-\w_\ast)].
\end{split}
\end{equation}

Because $\w_\ast$ is the optimum point of $\hat F_\delta(\w)$, there exist $\xi_\ast\in \partial h(\w_\ast)$ such that $\nabla f+\xi_\ast=0$. Therefore,
\begin{equation}\label{eq:bound_2}
\begin{split}
\hat{f}_\delta(\w)&-\hat{f}_\delta(\w_\ast)-\nabla\hat{f}_\delta(\w_\ast)(\w-\w_\ast)=\hat{f}_\delta(\w)-\hat{f}_\delta(\w_\ast)+\xi_\ast(\w-\w_\ast)\\
&\leq \hat{f}_\delta(\w)-\hat{f}_\delta(\w_\ast)+h(\w)-h(\w_\ast)=\hat F_\delta(\w)-\hat F_\delta(\w_\ast).
\end{split}
\end{equation}
The inequality uses the convexity of the $h(\w)$. We have the desired result from \eqref{eq:bound_1} and \eqref{eq:bound_2}.
\end{proof}
\section{Proof of Lemma \ref{lem:smooth}}
\begin{proof}
The smoothness implies that we have
  \begin{equation}\label{eq:smooth_1}
  f(\v)\leq f(\w)+\langle\nabla f(\w), \v-\w\rangle+\frac{L}{2}\|\v-\w\|^2.
  \end{equation}
 for all $\v, \w$. Setting $\v=\w-\frac{1}{L}\nabla f(\w)$ in the right-hand side of Eq. \eqref{eq:smooth_1} and rearranging terms, we obtain
 \begin{equation*}
  \frac{1}{2L}\|\nabla f(\w)\|^2\leq f(\w)-f(\v)\leq f(\w)-\min_\v f(\v).
 \end{equation*}
\end{proof}
\end{appendix}

%\begin{figure}[t]
%\includegraphics{}
%\caption{Figure caption.}\label{f1}
%\end{figure}

%\begin{table*}
%\caption{} \label{t1}
%\begin{tabular}{lll}
%\hline
%&&\\
%&&\\
%\hline
%\end{tabular}
%\end{table*}

%%%%%%%%%%%% The bibliography starts:
%\begin{thebibliography}{9}
%
%\bibitem{r1}
%
%\bibitem{r2}
%
%\end{thebibliography}

\end{document}